\documentclass[10pt,twocolumn,letterpaper]{article}

\usepackage{cvpr}
\usepackage{caption}
\usepackage{times}
\usepackage{epsfig}
\usepackage{graphicx}
\usepackage{amsmath}
\usepackage{amssymb}
\usepackage{amsthm}
\usepackage{lipsum}
\usepackage{enumerate}
\usepackage{pgfplots}
\usepackage{stmaryrd}
\usepackage{soul}
\usepackage[T1]{fontenc}
\usepackage[subrefformat=parens,labelformat=parens]{subcaption}
\usetikzlibrary{arrows.meta,decorations.markings}
\pgfplotsset{compat=newest}
\usepackage{colortbl}

\usepackage[pagebackref=true,breaklinks=true,letterpaper=true,colorlinks,bookmarks=false]{hyperref}

\cvprfinalcopy 

\def\cvprPaperID{190} 

\ifcvprfinal\fi

\newcommand{\lE}{\mathbf{E}}
\newcommand{\cC}{\mathcal{C}}
\newcommand{\cD}{\mathcal{D}}

\newcommand{\cK}{\mathcal{K}}
\newcommand{\cH}{\mathcal{H}}
\newcommand{\cO}{\mathcal{O}}
\newcommand{\cX}{\mathcal{X}}
\newcommand{\cL}{\mathcal{L}}
\newcommand{\cY}{\mathcal{Y}}
\newcommand{\cM}{\mathcal{M}}
\newcommand{\cG}{\mathcal{G}}

\newcommand{\Wh}{\mathbb{W}}
\newcommand{\dd}{\mathrm{d}}
\newcommand{\bM}{\mathbb{M}}
\newcommand{\Z}{\mathbf{Z}}
\newcommand{\R}{\mathbf{R}}
\newcommand{\Rext}{\overline{\mathbf{R}}}

\newcommand{\M}{\mathbf{M}}
\newcommand{\LM}{\mathbf{\Lambda}}
\newcommand{\spt}{\operatorname{spt}}
\newcommand{\cl}{\operatorname{cl}}

\newcommand{\norm}[1]{\| #1 \|}
\newcommand{\normc}[1]{| #1 |}
\newcommand{\iprod}[2]{\langle #1, #2 \rangle}
\newcommand{\bi}{\mathbf{i}}
\newcommand{\bj}{\mathbf{j}}
\newcommand{\measurerestr}{%
  \,\raisebox{-.127ex}{\reflectbox{\rotatebox[origin=br]{-90}{$\lnot$}}}\,%
}

\newtheorem{defi}{Definition}

\newtheorem{prop}{Proposition}

\newcommand{\id}{\operatorname{id}}

\definecolor{myorange}{rgb}{1,0.647058,0}

\begin{document}

\title{Lifting Vectorial Variational Problems: A Natural Formulation based on\\ Geometric Measure Theory and Discrete Exterior Calculus}

\author{Thomas M\"ollenhoff and Daniel Cremers\\
Technical University of Munich \\
{\tt\small \{thomas.moellenhoff,cremers\}@tum.de}
}

\maketitle

\begin{abstract}
  Numerous tasks in imaging and vision can be formulated as
  variational problems over vector-valued maps.  We approach the relaxation
  and convexification of such vectorial variational problems via a
  lifting to the space of currents. 
  To that end, we recall that functionals with polyconvex Lagrangians
  can be reparametrized as convex one-homogeneous functionals on the graph
  of the function. This leads to an equivalent shape optimization problem
  over oriented surfaces in the product space of domain and codomain.
  A convex formulation is then obtained by relaxing the search space from oriented surfaces to more general currents.
  We propose a discretization of the resulting infinite-dimensional
  optimization problem using Whitney forms, which also generalizes 
  recent ``sublabel-accurate'' multilabeling approaches.
\end{abstract}

\section{Introduction}
We consider functionals of $C^1$-mappings $f : \cX \to \cY$
\begin{equation}
  E(f) = \int_{\cX} c\left(x, f(x), \nabla f(x)\right) \mathrm{d} x,
  \label{eq:pc}
\end{equation}
where $\cX \subset \R^n$, $\cY \subset \R^N$ are bounded and open. The
cost function $c \equiv c(x, y, \xi)$ is assumed to be a nonnegative (possibly nonconvex) continuous
function on $\cX \times \cY \times \R^{N \times n}$ that is \emph{polyconvex} (see Def.~\ref{def:poly})
in the Jacobian matrix $\xi$.

This work is concerned with relaxation and global optimization of
\eqref{eq:pc} when, both, dimension and codimension are possibly larger than
one ($n> 1$, $N > 1$). This is expected to be difficult:
In the discrete setting problems with $n=1$ or $N=1$ typically correspond to polynomial-time solvable shortest path ($n=1$) or graph cut ($N=1$) problems~\cite{cohen1997global,tsitsiklis,Ishikawa,schoenemann2010combinatorial}, whereas for $n,N>1$, the arising
multilabel problems with unordered label spaces are known to be NP-hard - see \cite{li2016complexity}.
Nevertheless, heuristic strategies have been shown
to yield excellent results in tasks such as optical flow~\cite{fullflow} or shape matching~\cite{shekhovtsov2008efficient,chen2015robust}.
In contrast to such well-established Markov random field (MRF)
works~\cite{kolmogorov2006convergent,kolmogorov2007minimizing,kohli2008partial,shekhovtsov2008efficient,menze2015discrete,chen2015robust,fullflow,domokos2018mrf} we consider the way less explored
continuous (infinite-dimensional) setting.

Our motivation partly stems from the fact that formulations in function space are very
general and admit a variety of discretizations. Finite difference discretizations of
continuous relaxations often lead to models that are reminiscent of MRFs~\cite{Zach-et-al-cvpr12}, while piecewise-linear
approximations are related to discrete-continuous
MRFs~\cite{Zach-Kohli-eccv14},
see~\cite{fix2014duality,moellenhoff-iccv-2017}. More recently, for the Kantorovich
relaxation in optimal transport, approximations with deep neural networks
were considered and achieved promising performance, for example in generative
modeling~\cite{ACB17,seguy2017large}.

We further argue that fractional (non-integer) solutions to a careful discretization of the continuous model can implicitly approximate an ``integer'' continuous solution. Therefore one can achieve accuracies
that go substantially beyond the mesh size. The resulting models would be difficult to interpret and derive from a finite-dimensional viewpoint such that the continuous considerations are
required for the final implementation. Also, formulations arising from continuous relaxations allow one to introduce isotropic smoothness potentials without reverting to higher-order terms in
the cost, and, as we show in this work, one can impose general polyconvex regularizations using only local constraints.
An example of a polyconvex function (which is in general nonconvex) is the surface area of the graph,
sometimes referred to as ``Beltrami regularization'' in the image processing community, see e.g., \cite{belt}.

In contrast to the discrete multi-labeling setting, an important question is whether variational problems involving the energy \eqref{eq:pc} admit a minimizer.
A fruitful approach to address this question is to suitably relax the notion of solution, thereby enlarging the search space of admissible candidates (``lifting the problem to a larger space'').
The origins of this idea can be traced back\footnote{We refer the interested reader to the
historical remarks in L.~C.~Young's book on the calculus of
variations~\cite[pp.~122--123]{young2000lecture}.} to the turn of
the century, see Hilbert's twentieth problem~\cite{hilbert}.
An example of that principle is the celebrated Kantorovich relaxation~\cite{kantorovich1960} of Monge's transportation problem~\cite{monge1781}. There, the search over maps $f
: \cX \to \cY$ is relaxed to one over probability measures on the product space $\cX \times \cY$.
Each map can be identified in that extended space with a measure concentrated on its graph.
Existence of optimal transportation plans follows directly due to good compactness properties of the larger space.
Furthermore, the nonlinearly constrained and nonconvex optimization problem is transformed into one of linear programming, leading to rich duality theories and fast numerical algorithms~\cite{PeCu18}.

One may ask whether the relaxed solution in the extended
space has certain regularity properties, for example whether it is the graph of a
(sufficiently regular) map and thus can be considered a solution
to the original (``unlifted'') problem. In the case of optimal
transport, such regularity theory can be guaranteed under some assumptions~\cite{Vil08,San15}. Establishing existence and regularity for
problems in which the cost additionally depends on the Jacobian (for example minimal surface problems) has been a driving factor in
the development of geometric measure theory, see~\cite{morgan2016geometric} for an introduction.
In this work, we will use ideas from geometric measure theory to pursue the above relaxation and lifting principle for the energy~\eqref{eq:pc}.
The main idea is to reformulate the original variational problem as a shape optimization
problem over oriented manifolds representing the graph of the map $f : \cX \to \cY$ in the product space $\cX \times \cY$ .
To obtain a convex formulation we enlarge the search space from oriented manifolds to currents.

\subsection{Related Work}
A common strategy to solve problems involving \eqref{eq:pc} is to revert to local gradient descent minimization based on the Euler-Lagrange equations.
But for nonconvex problems solutions might depend on the initialization and the computed stationary points may be quite suboptimal.
Therefore, we pursue the aforementioned lifting of the energy \eqref{eq:pc} to currents. This lifting has been
previously considered in geometric measure theory to establish the aforementioned existence and regularity theory for vectorial variational problems
in a very broad setting, see e.g.,~\cite{federer,federer1974real,aviles1991variational}.
In contrast to such impressive theoretical achievements, this paper is concerned with a discretization and implementation.

There is also a variety of related applied works.
The paper~\cite{windheuser2011geometrically} tackles the problem of bijective and smooth shape matching using linear programming. 
Similar to the present work, the authors also look for graph surfaces in $\cX \times \cY$ but they consider the discrete setting and use a different notion of boundary operator.
We study the continuous setting, but also our discrete formulation is quite different. 

For $N=1$, the proposed continuous formulation specializes to~\cite{ABDM,PCBC-SIIMS}. To tackle the setting of $N > 1$ in a memory efficient manner,
Strekalovskiy~\etal~\cite{Strekalovskiy-et-al-cvpr12,goldluecke2013tight,strekalovskiy-et-al-siims14} keep a collection of $N$ surfaces with codimension one under the factorization assumption that
$\cY = \cY_1 \times \hdots \times \cY_N$. In contrast, we consider only one surface of codimension $N$, we do not require an assumption on $\cY$, our approach is applicable to a larger class of
functionals and we expect it to yield a tighter relaxation. The lifting approaches \cite{lellmann-et-al-iccv2013,goldstein2012global} also tackle vectorial problems by considering the full product
space, but are limited to total variation regularization (with the former allowing $\cY$ to be a manifold). The recent work \cite{windheuser2016convex} is most related to the present one, however
their relaxation considers a specific instance of~\eqref{eq:pc}. Moreover, the above works are based on finite difference discretizations of the continuous model. In contrast, the proposed
discretization using discrete exterior calculus yields solutions beyond the mesh accuracy as in recent sublabel-accurate approaches. The latter are restricted to $N=1$~\cite{moellenhoff-laude-cvpr-2016,moellenhoff-iccv-2017} or total variation regularization~\cite{laude16eccv}. Recent works also include extensions to total generalized variation or Laplacian regularization \cite{strecke2018sublabel,vogt,loewe}.

Recent approaches in shape analysis~\cite{solomon2016entropic,vestner2017product,vestner2017efficient}
also operate in the product space $\cX \times \cY$. However, these are based on
local minimizations of the Gromov-Wasserstein distance~\cite{memoli2007use}
and spectral variants thereof~\cite{memoli2009spectral} which leads
to (nonconvex) quadratic assignment problems. While the goal to find a 
smooth (possibly bijective) map is similar, the formulations appear to be quite different.
To alleviate the increased cost of the product space formulation,
computationally efficient representations of densities in $\cX \times \cY$ have been
studied in the context of functional maps \cite{ovsjanikov2012functional, rodola2018functional}.

\section{Notation and Preliminaries}
\label{sec:fundamentals}
Throughout this paper we will introduce notions from geometric measure theory, as they are
not commonly used in the vision community. While the subject is rather technical,
our aim is to keep the presentation light and to focus on the geometric intuition and
aspects which are important for a practical implementation. We invite the reader
to consult chapter 4 in the book~\cite{morgan2016geometric} and the chapter on
exterior calculus in~\cite{crane2015discrete}, which both contain many illuminating illustrations.
For a more technical treatment we refer to~\cite{federer,KP08}.

In the following, we denote a basis in $\R^d$ as $\{ e_1, \hdots, e_d \}$ with dual basis $\{ \dd x_1, \hdots, \dd x_d \}$ where $\dd x_i : \R^d \to \R$ is
the linear functional that maps every $x = (x_1, \hdots, x_d)$ to the $i$-th component $x_i$.
Given an integer $k \leq d$, $I(d, k)$ are the multi-indices $\bi = (i_1, \hdots, i_k)$ with $1 \leq i_1 < \hdots < i_k \leq d$.

As we will consider $n$-surfaces in $\cX \times \cY \subset \R^{n+N}$, most of the time we set $d = n + N$ and $k = n$. To further simplify notation, we denote the basis vectors $\{ e_{n+1}, \hdots, e_{n+N} \}$ by $\{ \varepsilon_1, \hdots, \varepsilon_{N} \}$ and similarly refer to the dual basis as $\{ \dd x_1, \hdots \dd x_n, \dd y_1, \hdots, \dd y_N \}$.
When it is clear from the context, we treat vectors $e_i \in \R^{n}$ and $\varepsilon_i \in \R^N$
in the sense that $e_i \simeq (e_i, \mathbf{0}_N) \in \R^{n+N}$, $\varepsilon_i \simeq (\mathbf{0}_n, \varepsilon_i) \in \R^{n+N}$. As an example, for $\nabla f(x) \in \R^{N \times n}$ we can define the expression $e_i + \nabla f(x) e_i$ and read it as $\left( e_i, \nabla f(x) e_i \right) \in \R^{n + N}$.

\subsection{Convex Analysis}
The extended reals are denoted by $\Rext = \R \cup \{ +\infty \}$. For a finite-dimensional real vector space $V$ and $\Psi : V \to \Rext$ we denote the
convex conjugate as $\Psi^* : V^* \to \Rext$ and the biconjugate as $\Psi^{**} : V \to \Rext$. $\Psi^{**}$ is the largest lower-semicontinuous convex function below $\Psi$.
In our notation, for functions with several arguments, the conjugate is always taken only in the last argument. As a general reference to convex analysis, we refer the reader to the books \cite{hiriart2012fundamentals,Rockafellar:ConvexAnalysis}.

\subsection{Multilinear Algebra}
The formalism of multi-vectors we introduce in this section is central to this work, as the idea of the relaxation is to represent the oriented graph of $f$ by a $k$-vectorfield (more precisely: a
$k$-current) in the product space $\cX \times \cY$.
Basically, one can multiply $v_i \in \R^d$ to obtain an object
\begin{equation}
  v = v_1 \wedge \hdots \wedge v_k,
  \label{eq:sv}
\end{equation}
called a \emph{simple} $k$-vector in $\R^d$. The geometric intuition of simple $k$-vectors is, that they describe the $k$-dimensional space spanned by the $\{ v_i \}$, together with an orientation and the
area of the parallelotope given by the $\{ v_i \}$. Thus, simple $k$-vectors can be thought of oriented parallelotopes as shown in orange in Fig.~\ref{fig:graph_illu}.
In general, $k$-vectors are defined to be formal sums
\begin{equation}
  v = \sum_{\bi \in I(d, k)} v^{\bi} \cdot e_{\bi_1} \wedge \hdots \wedge e_{\bi_k} = \sum_{\bi \in I(d,k)} v^{\bi} \cdot e_{\bi},
\end{equation}
for coefficients $v^{\bi} \in \R$. They form the vector space $\LM_k \R^d$, which has dimension $\binom{d}{k}$.

The dual space $\LM^k \R^d$ of $k$-covectors is defined analogously, with $\iprod{\dd x_\bi}{e_\bj} = \delta_{\bi\bj}$.
We define for two $k$-vectors (and also for $k$-covectors) $v = \sum_\bi v_\bi e_\bi$, $w = \sum_\bi w_\bi e_\bi$ an inner product $\iprod{v}{w} = \sum_\bi v_{\bi} w_{\bi}$ and norm $\normc{v} = \sqrt{\iprod{v}{v}}$.

$k$-vectors (elements of $\LM_k \R^d$) are called \emph{simple}, if they can be written as the \emph{wedge product}
of $1$-vectors as in \eqref{eq:sv}.
Unfortunately, for $1 < k < d - 1$, not all $k$-vectors are simple and the set of simple $k$-vectors is a nonconvex cone in $\LM_k \R^d$, called the Grassmann cone~\cite{busemann1963convex}.
This is one aspect why the setting of $n > 1$ and $N > 1$ is more challenging. 

Later on, we will consider a relaxation from the nonconvex set of simple $k$-vectors 
to general $k$-vectors. Naturally, for the relaxation to be good, we want the convex energy to be 
\emph{as large as possible} on non-simple $k$-vectors.
For the Euclidean norm, a good convex extension is the \emph{mass} norm
\begin{equation}
  \norm{v} = \inf \left \{ \sum_i |\xi_i| : \xi_i \text{ are simple}, v = \sum_i \xi_i \right \}.
  \label{eq:mass}
\end{equation}
The dual norm is the \emph{comass} norm given by:
\begin{equation}
  \norm{w}^* = \sup \left \{ \iprod{w}{v} : v \text{ is simple }, |v| \leq 1 \right \}.
  \label{eq:comass}
\end{equation}
The mass norm can be understood as the largest norm that agrees with the Euclidean norm on simple $k$-vectors.

\graphicspath{{figures/illustrations/}}
\begin{figure}[t!]
  \centering
  \def\svgwidth{\linewidth}
  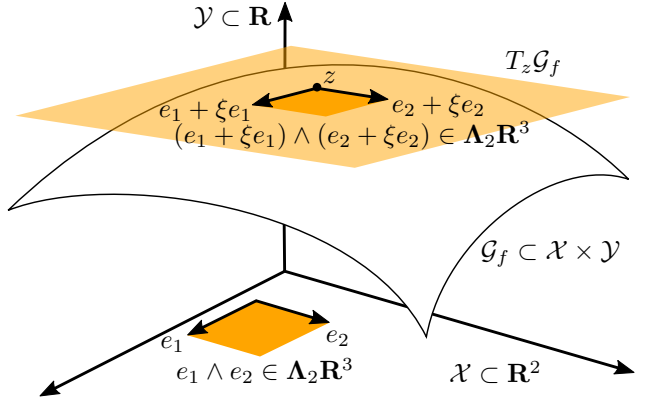
  \vspace{0.01mm}
  \caption{Illustration for the setting of $n = 2$, $N = 1$. The graph $\cG_f$ of the $C^1$-map $f : \cX \to \cY$ is a smooth oriented manifold embedded in the product space $\cX \times \cY$. The
    tangent space at $z = (x, f(x))$ is spanned by the simple $n$-vector $(e_1 + \xi e_1) \wedge \hdots \wedge (e_n + \xi e_n) \in \LM_n \R^{n+N}$,
    where $\xi = \nabla f(x) \in \R^{N \times n}$ is the Jacobian.}
  \label{fig:graph_illu}
\end{figure}

\section{Lifting to Graphs in the Product Space}
\label{sec:product-space}
With the necessary preliminaries in mind, our goal is now to reparametrize the original energy \eqref{eq:pc}
to the graph $\cG_f \subset \cX \times \cY$.
As shown in Fig.~\ref{fig:graph_illu}, the graph is an oriented $n$-dimensional
manifold in the product space with global parametrization $u(x) = (x, f(x))$.
\begin{defi}[Orientation]
If $\cM \subset \R^d$ is a $k$-dimensional smooth manifold in $\R^d$ (possibly with boundary), an \textbf{orientation} of $\cM$ is a
continuous map $\tau_\cM : \cM \to \LM_k \R^d$ such that $\tau_\cM(z)$ is a simple $k$-vector with unit norm
that spans the tangent space $T_z \cM$ at every point $z \in \cM$.
\end{defi}
From differential geometry we know that the tangent space $T_z \cG_f$ at $z = (x, f(x))$ is spanned by
$\partial_{i} u (u^{-1}(z)) = e_i + \nabla f(x) e_i$.
Therefore, an orientation of $\cG_f$ is given by
\begin{equation}
  \tau_{\cG_f}(z) = \frac{M(\nabla f(\pi_1 z))}{\normc{M(\nabla f(\pi_1 z))}},
  \label{eq:tau}
\end{equation}
where the map $M : \R^{N \times n} \to \LM_n \R^{n+N}$ is given by
\begin{equation}
  M(\xi) = (e_1 + \xi e_1) \wedge \hdots \wedge (e_n + \xi e_n),
  \label{eq:map_M}
\end{equation}
and $\pi_1 : \cX \times \cY \to \cX$ is the canonical projection onto the first argument.
In order to derive the reparametrization, we have to connect a simple $n$-vector (representing an oriented tangent plane of the graph)
with the Jacobian of the original energy. For that, we need an inverse of the map given in \eqref{eq:map_M}.

To derive such an inverse, we first introduce further helpful notations. For $\bi \in I(m, l)$ we denote by $\bar{\bi} \in I(m, m - l)$ the
element which complements $\bi$ in $\{1, 2, \hdots, m \}$ in increasing order, denote $\bar{0} = \{1, \hdots, m \}$ and $0$ as the 
empty multi-index.
Every $v \in \LM_n \R^{n+N}$ can be written as
\begin{equation}
  v = \sum_{|\bi| + |\bj| = n} v^{\bi, \bj} e_\bi \wedge \varepsilon_\bj,
  \label{eq:decomp1}
\end{equation}
where $\bi \in I(n, l)$, $\bj \in I(N, l')$, $l + l' = n$. 
To give an example, the $\binom{5}{2}=10$ coefficients of a $2$-vector $v \in \LM_2 \R^5$ according to
the notation \eqref{eq:decomp1} are:
\newcommand{\ckg}{\cellcolor{gray!10}}
\begin{equation}
  \begin{matrix}
    v^{\bar 0,0} & & & \\
    \ckg v^{1,1} & \ckg v^{2,1} & & \\
    \ckg v^{1,2} & \ckg v^{2,2} & v^{0,(1,2)} & \\
    \ckg v^{1,3} & \ckg v^{2,3} & v^{0,(1,3)} & v^{0,(2,3)},
  \end{matrix}
\end{equation}
where we highlighted the \colorbox{gray!10}{$N \times n$ coefficients} with $|\bj| = 1$. 
Now note that the vector $v = M(\xi)$ is by construction a simple $n$-vector with first component $v^{\bar0,0} = 1$. To any
$v \in \LM_n \R^{n+N}$ with $v^{\bar 0,0} = 1$ we associate $\xi(v) \in \R^{N \times n}$ given by
\begin{equation}
  \left[ \xi(v) \right]_{j,i} = (-1)^{n - i} v^{\bar i,j}.
  \label{eq:inverse}
\end{equation}
If and only if $v \in \LM_n \R^{n+N}$ is simple with first component $v^{\bar 0,0} = 1$ then $v = M(\xi(v))$. A proof is given in~\cite[Vol.~I, Ch.~2.1, Prop.~1]{GMS-CC}. Thus, 
on the set of simple $n$-vectors with first component $v^{\bar 0,0} = 1$,
\begin{equation}
    \Sigma_1 = \{ v \in \LM_n \R^{n + N} : v = M(\xi) \text{ for } \xi \in \R^{N \times n} \},
\end{equation}
the inverse of the map \eqref{eq:map_M} is given by \eqref{eq:inverse}.

Using the above notations, we can define a generalized notion of convexity, which
essentially states that there is a convex reformulation on $k$-vectors.
\begin{defi}[Polyconvexity]
  A map $c : \R^{N \times n} \to \Rext$ is \textbf{polyconvex} if there is a convex function
  $\bar c : \LM_n \R^{n+N} \to \Rext$ such that we have
  \begin{equation}
    c(\xi) = \bar c(M(\xi)) ~~\text{ for all }~~ \xi \in \R^{N \times n}.
    \label{eq:pce}
  \end{equation} 
  Equivalently one has that $c(\xi(v)) = \bar c(v)$ for all $v \in \Sigma_1$. We also refer to the convex function $\bar c$ as a \textbf{polyconvex extension}.
  \label{def:poly}
\end{defi}
In general, the polyconvex extension is not unique. Any convex function has an obvious polyconvex extension by \eqref{eq:inverse}, but
as discussed in the previous section we would like the convex extension to be as large as possible for $v \notin \Sigma_1$. 
The largest polyconvex extension which agrees with the original function on $\Sigma_1$ can be 
formally defined using the convex biconjugate, but is often hard to explicitly compute. The mass norm \eqref{eq:mass} corresponds to such a construction.

Nevertheless, given any polyconvex extension, we can now reparametrize the original energy \eqref{eq:pc} on the oriented graph $\cG_f$, as we show in the following central proposition.
\begin{prop}
  Let $\bar c : \cX \times \cY \times \LM_n \R^{n + N} \to \Rext$ be a polyconvex extension of the original cost $c$ in the last argument.
  Define the function $\Psi : \cX \times \cY \times \LM_n \R^{n + N} \to \Rext$,
  \begin{equation}
    \Psi(z, v) =
    \begin{cases}
      v^{\bar 0, 0} \bar c(\pi_1 z,  \pi_2 z, v / v^{\bar 0, 0}), &\text{if } v^{\bar 0,0} > 0,\\
      +\infty, &\text{otherwise,}
    \end{cases}
    \label{eq:phi}
  \end{equation}
  where $\pi_1 : \cX \times \cY \to \cX$ and $\pi_2 : \cX \times \cY \to \cY$ are the canonical projections
  onto the first and second argument.
  Then we can reparametrize \eqref{eq:pc} as follows:
  \begin{equation}
    \begin{aligned}
      &\int_{\cX} c(x, f(x), \nabla f(x)) \, \dd \cL^n(x) \\
      &\qquad = \int_{\cG_f} \Psi(z, \tau_{\cG_f}( z )) \, \dd \cH^n(z),
    \end{aligned}
    \label{eq:graph}
  \end{equation}
  where the second integral is the standard Lebesgue integral with respect to the $n$-dimensional Hausdorff measure on $\R^{n+N}$ restricted
  to the graph $\cG_f$. 
  \label{prop1}
\end{prop}
\begin{proof}
  We directly calculate:
  \begin{align}
      &\int_{\cX} c\left( x, f(x), \nabla f(x) \right) \dd \cL^n(x) \label{eq:prf_start} \\
      &\, {=} \int_{\cX} \Psi \left( x, f(x), M(\nabla f(x)) \right) \dd \cL^n(x) \label{eq:prfa}\\
      &\, {=} \int_{\cG_f} \Psi \left( z, M(\nabla f(\pi_1 z)) \right) \frac{1}{|M(\nabla f(\pi_1 z))|} \mathrm{d}\cH^n(z) \label{eq:prfb} \\
      &\, {=} \int_{\cG_f} \Psi \left( z, \tau_{\cG_f}(z) \right) \mathrm{d}\cH^n(z). \label{eq:prfc}
    \end{align}
  The step from \eqref{eq:prf_start} to \eqref{eq:prfa} uses 
  that $\bar c$ is a polyconvex extension (so that we can apply \eqref{eq:pce}) and the fact that for $v = M(\nabla f(x))$ we have $v^{\bar 0,0} = 1$. To arrive at \eqref{eq:prfb}, an
  application of the area formula~\cite[Corollary~5.1.13]{KP08} suffices and
  for \eqref{eq:prfc} we used positive one-homogenity of $\Psi$ and the definition of $\tau_{\cG_f}$ in \eqref{eq:tau}.
\end{proof}
Interestingly, the function \eqref{eq:phi} is convex and one-homogeneous
in the last argument, as it is the \emph{perspective} of a convex function.
However, the search space of oriented graphs of $C^1$ mappings is nonconvex. Therefore we relax from
oriented graphs to the larger set of currents, which we will introduce in the following section.
Since currents form a vector space, we therefore obtain a convex functional over a convex domain.

\section{From Oriented Graphs to Currents}
\begin{figure*}[t!]
  \centering
  \begin{center}
  \begin{subfigure}{0.245\linewidth}
    \def\svgwidth{1\linewidth}
    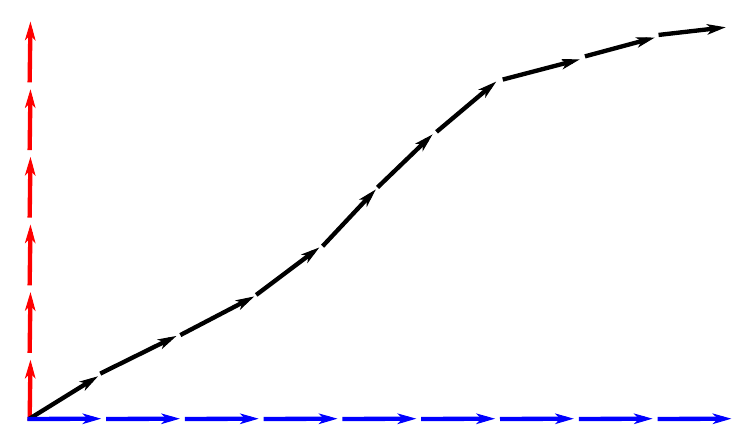
    \caption{Graph of diffeomorphism $f$}
    \label{fig:a}
  \end{subfigure}
  \begin{subfigure}{0.245\linewidth}
    \def\svgwidth{1\linewidth}
    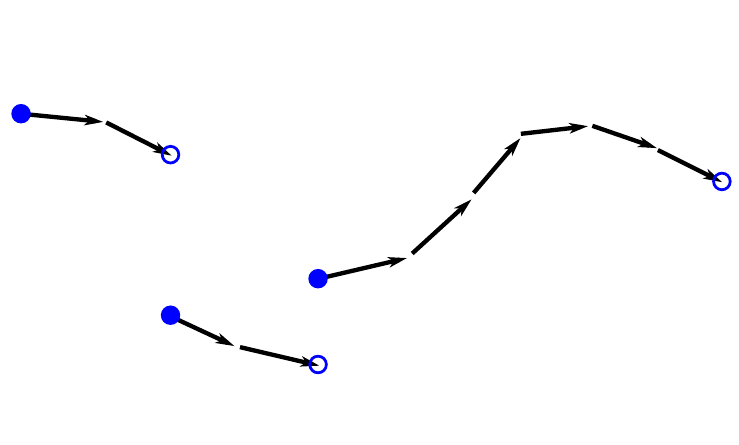
    \caption{Graph of function with jumps}
    \label{fig:b}
  \end{subfigure}
  \begin{subfigure}{0.245\linewidth}
    \def\svgwidth{1\linewidth}
    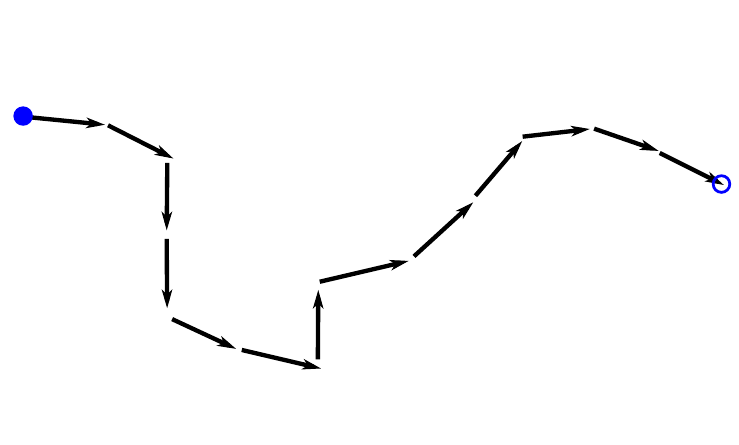
    \caption{``Stitched'' graph}
    \label{fig:c}
  \end{subfigure}
  \begin{subfigure}{0.245\linewidth}
    \def\svgwidth{1\linewidth}
    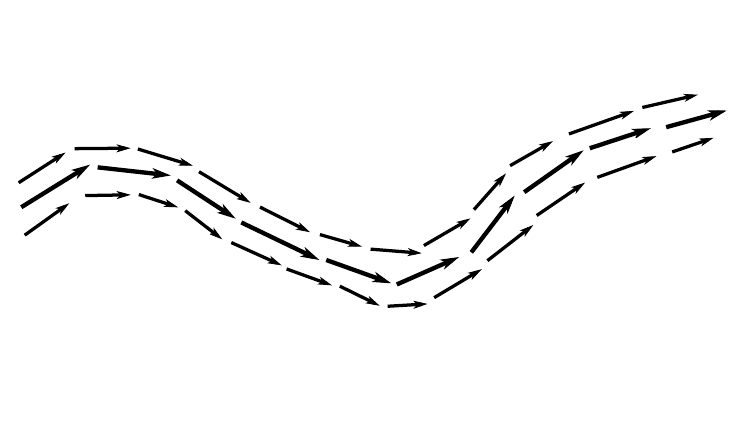
    \caption{Current which is not a graph}
    \label{fig:d}
  \end{subfigure}
  \end{center}
  \caption{The idea of our relaxation is to move from oriented graphs in the product space to the larger set of currents. These include oriented graphs as special cases, as shown in
     \subref{fig:a}. For a diffeomorphism, the pushforwards $\pi_{1\sharp} T$, $\pi_{2\sharp} T$ yield currents induced by domain and codomain, which will be a linear constraint in the relaxed problem.
     In \subref{fig:b} we show the current given by the graph of a discontinuous function. Since it has holes, the boundary operator $\partial T$ has support inside the domain. We will constrain
     the support of the boundary to exclude such cases. \subref{fig:c} Stitching jumps yields a current with vertical parts at the jump points,
     which corresponds to the limiting case in the perspective function \eqref{eq:phi}. To obtain an overall convex formulation, we will also allow currents \subref{fig:d} which don't
     necessarily concentrate on the graph of a function.}
\label{fig:currents}
\end{figure*}
\label{sec:lifting}

Throughout this section, let $U \subset \R^d$ be an open set, which will later be a neighbourhood of $X \times Y \subset \R^{n+N}$, where $X = \cl(\cX)$, $Y = \cl(\cY)$ are the closures of $\cX, \cY$.
The main idea of our relaxation and the geometric intuitions of \emph{pushforward} and \emph{boundary} operator
we introduce in this section are summarized in the following Fig.~\ref{fig:currents}. Currents are defined
in duality with differential forms, which we will briefly introduce in the following section.

\subsection{Differential Forms}
\label{sec:cur}
A differential form of order $k$ (short: $k$-form) is a map $\omega : U \to \LM^k \R^d$.
The \emph{support} of a differential form $\spt \omega$ is
defined as the closure of $\{ z \in U : \omega(z) \neq 0 \}$.
Integration of a $k$-form over an oriented $k$-dimensional manifold is defined by
\begin{equation}
  \int_\cM \omega := \int_\cM \iprod{\omega(z)}{\tau_\cM(z)} \, \mathrm{d} \cH^k(z).
  \label{eq:integral}
\end{equation}
A notion of derivative for $k$-forms is the exterior derivative $d \omega$, which is a $(k+1)$-form
given by:
\begin{equation}
  \begin{aligned}
    &\iprod{d \omega(z)}{v_1 \wedge \hdots \wedge v_{k+1}} =
    \lim_{h \to 0} \frac{1}{h^{k+1}} \int_{\partial P} \omega, 
  \end{aligned}
  \label{eq:exterior_diff}
\end{equation}
where $\partial P$ is the oriented boundary of the parallelotope spanned by the $\{ h v_i \}$ at point $z$.

To get an intuition, note that for $k = 0$ this reduces to the familiar directional derivative $\iprod{d\omega(x)}{v_1} = \lim_{h \to 0} \frac{1}{h} \left( \omega(x + h v_1) - \omega(x) \right)$.
With \eqref{eq:integral} and \eqref{eq:exterior_diff} in mind, one sees why Stokes' theorem
\begin{equation}
  \int_{\cM} d \omega = \int_{\partial \cM} \omega.
  \label{eq:stokes}
\end{equation}
should hold intuitively. Given a map $\pi : \R^d \to \R^q$, the \emph{pullback} $\pi^\sharp \omega$ of the $k$-form $\omega$
is determined by
\begin{equation}
    \iprod{\pi^\sharp \omega}{v_1 \wedge .. \wedge v_k} = \iprod{\omega \circ \pi}{D_{v_1} \pi \wedge .. \wedge D_{v_k} \pi},
\end{equation}
where $D_{v_i} \pi = \nabla \pi \cdot v_i$ and $\nabla \pi \in \R^{q \times d}$ is the Jacobian. 

\subsection{Currents}
Denote the space of smooth $k$-forms with compact
support on $U$ as $\cD^k(U)$. \emph{Currents} are elements of the dual space
$\cD_k(U) = \cD^k(U)'$, i.e., linear functionals acting on differential forms.
As shown in Fig.~\ref{fig:a}, an oriented $k$-dimensional manifold $\cM \subset U$ induces
a current by 
\begin{equation}
  \llbracket \cM \rrbracket (\omega) = \int_{\cM} \omega.
\end{equation}
However, since $\cD_k(U)$ is a vector space, not all elements look like $k$-dimensional manifolds, see Fig.~\ref{fig:d}.
The \emph{boundary} of the $k$-current $T \in \cD_k(U)$ is the $(k - 1)$-current $\partial T \in \cD_{k-1}(U)$ defined via the exterior derivative:
\begin{equation}
  \partial T(\omega) = T(d \omega), \, \text{ for all } \, \omega \in \cD^{k-1}(U).
  \label{eq:bdry_cont}
\end{equation}
Stokes' theorem \eqref{eq:stokes} ensures that
for currents which are given by $k$-dimensional oriented manifolds, the boundary of the current agrees with the usual notion,
see Fig.~\ref{fig:b}.

The \emph{support} of a current, denoted by $\spt T$, is the complement of the biggest open set $V$ such that
\begin{equation}
  T(\omega)  = 0~ \text{ whenever }~ \spt(\omega) \subset V.
\end{equation}
  Given a map $\pi : \R^d \to \R^q$ the \emph{pushforward} $\pi_\sharp T$ of
  the $k$-current $T \in \cD_k(U)$ is given by
  \begin{equation}
    \pi_\sharp T(\omega) = T(\pi^\sharp \omega), \, \text{ for all }\, \omega \in \cD^k(\R^q).
  \end{equation}
  Intuitively, it transforms the current using the map $\pi$, as illustrated in Fig.~\ref{fig:a}.
  The \emph{mass} of a current $T \in \cD_k(U)$ is
  \begin{equation}
    \bM(T) = \sup \left \{ T(\omega) : \omega \in \cD^k(U), \norm{\omega(z)}^* \leq 1 \right\},
  \end{equation}
  and as expected $\bM(\llbracket \cM \rrbracket) = \cH^k(\cM)$.
We denote the space of $k$-currents with finite mass and compact support by $\M_k(U)$.
These are \emph{representable by integration}, meaning there is a
measure $\norm{T}$ on $U$ and a map $\vec{T} : U \to \LM_k \R^d$ such that $\norm{\vec{T}(z)} = 1$ for $\norm{T}$-almost all $z$ such that
\begin{equation}
  T(\omega) = \int \iprod{\omega(z)}{\vec{T}(z)} \, \mathrm{d} \norm{T}(z).
  \label{eq:polar}
\end{equation}
The decomposition \eqref{eq:polar} is crucial, and we will use it to define the relaxation
in the next section.

\subsection{The Relaxed Energy}
We lift the original energy \eqref{eq:pc} to the space of
finite mass currents $T \in \M_n(U)$ with $\spt T \subset X \times Y$ as follows:
\begin{equation}
  \lE(T) = \int \Psi^{**} \left( \pi_1 z, \pi_2 z, \vec{T}(z) \right) \, \mathrm{d} \norm{T}(z).
  \label{eq:lift_en}
\end{equation}
Since for $T = \llbracket \cG_f \rrbracket$ we have $\vec{T} = \tau_{\cG_f}$, $\norm{T} = \cH^n \measurerestr \cG_f$ the desirable property $\lE(\llbracket \cG_f \rrbracket) = E(f)$
holds due to Prop.~\ref{prop1}.

Note that in \eqref{eq:lift_en} we use the lower-semicontinuous regularization $\Psi^{**}$
which extends \eqref{eq:phi} at $v^{\bar 0,0} = 0$
with the correct value.
Interestingly, this point corresponds
to the situation when the graph has vertical parts, which cannot
occur for $C^1$ functions but can happen for general currents, see~Fig.~\ref{fig:c}.
In \cite{Mora} it was shown that one can penalize such jumps in a way depending on the jump distance and direction.
We will not consider such additional regularization due to space limitations, but remark that they could be integrated by adding
further constraints to the following dual representation, which is a consequence of \cite[Vol. II, Sec. 1.3.1, Thm.~2]{GMS-CC}.
\begin{prop}
\label{prop2}
For $T \in \M_n(U)$ with $\spt T \subset X \times Y$, we have the dual representation
\begin{equation}
  \lE(T) = \sup_{\omega \in \cK} ~ T(\omega),
  \label{eq:dual_en}
\end{equation}
where the constraint is the closed and convex set
\begin{equation}
  \begin{aligned}
    \cK = \Bigl\{ &\omega \in C_c^0(U, \LM^n \R^{n+N}) :  \\
     &\Psi^*(\pi_1 z, \pi_2 z, \omega(z)) \leq 0, \forall z \in X \times Y \Bigr\}.
  \end{aligned}
\end{equation}
\end{prop}
The final relaxed optimization problem for \eqref{eq:pc} reads
\begin{equation}
  \begin{aligned}
    &\inf_{T \in \M_n(U)} ~ \lE(T), ~ \text{ s.t. } ~ \spt T \subset X \times Y, ~ T \in \cC.
  \end{aligned}
  \label{eq:pv}
\end{equation}
Depending on the kind of problem one wishes to solve, a different convex constraint set $\cC$ should be considered.
For example, in the case of variational problems with Dirichlet boundary conditions, we set
\begin{equation}
  \begin{aligned}
    \cC = \bigl \{ T ~:~ \pi_{1\sharp} T = \llbracket X \rrbracket, ~ \partial T = S \bigr \},
  \end{aligned}
\end{equation}
where $S \in \M_{n-1}(U)$ is a given boundary datum. In case of Neumann boundary conditions,
one constrains the support of the boundary to be zero inside the domain
\begin{equation}
  \begin{aligned}
    \cC = \bigl \{ T ~:~ \pi_{1\sharp} T = \llbracket X \rrbracket, \, \spt \partial T \subset (\partial X) \times Y \bigr \},
  \end{aligned}
\end{equation}
to exclude surfaces with holes, but allow the boundary to be freely chosen on $(\partial X) \times Y$.
In case $n = N$, one can also consider the constraint set
\begin{equation}
  \begin{aligned}
    \cC = \bigl \{ T ~:~ &\pi_{1\sharp} T = \llbracket X \rrbracket, \pi_{2\sharp} T = \llbracket Y \rrbracket, \\
    &\spt \partial T \subset \partial (X \times Y) \bigr \},
  \end{aligned}
  \label{eq:bij}
\end{equation}
where the additional pushforward constraint encourages bijectivity.
Notice also the similarity of \eqref{eq:pv} together with \eqref{eq:bij} to the Kantorovich relaxation in optimal transport.

Existence of minimizing currents to a similar problem as \eqref{eq:pv} in a certain space of currents (\emph{real flat chains}) is shown in \cite[\textsection3.9]{federer1974real}.
For dimension $n = 1$ or codimension $N = 1$, the
infimum is actually realized by a surface (\emph{integral flat chain}) \cite[\textsection5.10, \textsection5.12]{federer1974real}.
An adaptation of such theoretical considerations to our setting and conditions under which the relaxation is tight in the scenario $n > 1$, $N > 1$ is a major open challenge and left for future work.

\section{Discrete Formulation}
\label{sec:implementation}
In this section we present an implementation of the continuous model
\eqref{eq:pv} using discrete exterior calculus~\cite{Hirani2003}. 
We will base our discretization on cubes since they are easy to work with in high dimensions,
but one could also use simplices. To define cubical meshes, we adopt some notations from computational homology~\cite{CH}.
\begin{defi}[Elementary interval and cube]
  An \textbf{elementary interval} is an interval $I \subset \R$ of the form
  $I = [l, l + 1]$ or $I = \{ l \}$ for $l \in \Z$. Intervals that consist of a single point are \textbf{degenerate}.
  An \textbf{elementary cube} is given by a product 
  $\kappa = I_1 \times \hdots \times I_d$,
  where each $I_i$ is an elementary interval. The set of elementary cubes in $\R^d$ is denoted by $K^d$.
\end{defi}
For $\kappa \in K^d$, denote by $\dim \kappa \in \{ 1, \hdots, d \}$ the number of nondegenerate intervals. We  denote
$\bi(\kappa) \in I(d, \dim \kappa)$ as the multi-index referencing the nondegenerate intervals.
\begin{defi}[Cubical set]
A set $Q \subset \R^d$ is a \textbf{cubical set} if it can be written as a finite union of elementary cubes. 
\end{defi}
Let $K_k^d(Q) = \{ \kappa \in K^d : \kappa \subset Q, \dim \kappa = k \}$
be the set of $k$-dimensional cubes contained in $Q \subset \R^d$.
A map $\phi : Q \to X \times Y$ will transform the cubical set to our domain.
As we work with images, it will just be a mesh spacing, i.e., we set $\phi(z) = (h_1 z_1, \hdots, h_d z_d)$.

\begin{defi}[$k$-chains, $k$-cochains]
  We denote the space of finite formal sums of elements in $K_k^d(Q)$ with real coefficients as $\cC_k(Q)$, called (real) \textbf{$k$-chains}. We denote the dual as $\cC_k(Q)^* = \cC^k(Q)$ and call the elements \textbf{$k$-cochains}. 
\end{defi}

\begin{figure}[t]
  \centering
  \input{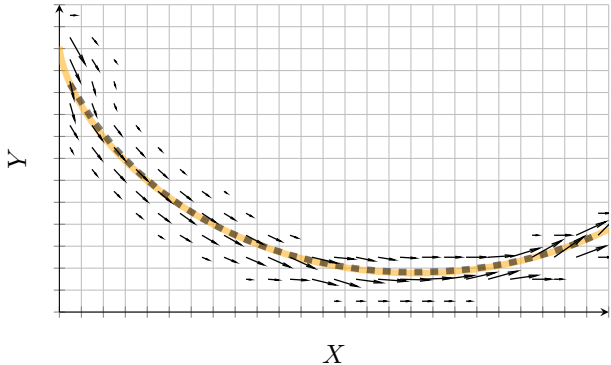}
  \caption{Minimization of the \emph{Brachistochrone} energy on a $25 \times 14$ cubical set (gray squares). The proposed discretization yields a diffuse current (black vector field), whose center of mass (black, dashed) however is faithful to the analytical cycloid solution (orange) far beyond the mesh accuracy.}
  \label{fig:illu}
\end{figure}

\begin{defi}[Boundary]
  For $\kappa \in K_k^d(Q)$, denote the primary faces obtained by collapsing the $j$-th 
  non-degenerate interval to the lower respectively upper boundary as
  $\kappa_j^-, \kappa_j^+ \in K_{k-1}^d$.
  The \textbf{boundary} of an elementary cube $\kappa \in K_k^d(Q)$ is the $(k-1)$-chain, 
  \begin{equation}
    \partial \kappa = \sum_{j=1}^k (-1)^{j-1} (\kappa_j^+ - \kappa_j^-) \in \cC_{k-1}(Q).
  \end{equation}
  The \textbf{boundary operator} is given by the extension to a linear map $\partial : \cC_{k}(Q) \to \cC_{k-1}(Q)$. 
\end{defi}
A $k$-chain $T = \sum_\kappa T_\kappa \kappa \in \cC_k(Q)$ can be identified with a $k$-current $T' \in \cD_k(U)$ by $T' = \sum_{\kappa} T_\kappa \phi_\sharp \llbracket \kappa \rrbracket$. 
The above discrete notion of boundary is defined in analogy to the continuous definition \eqref{eq:bdry_cont}.

In our discretization, we will use the dual representation of the lifted energy from Prop.~\ref{prop2}. To implement
differential forms, we introduce an interpolation operator.
\begin{defi}[Whitney map]
  The \textbf{Whitney map} extends a $k$-cochain $\omega$ to a $k$-form $(\Wh \omega) : X \times Y \to \LM^k \R^d$,
  \begin{equation}
    (\Wh \omega)(x) = \sum_{\kappa \in K^d_k(Q)} \omega_\kappa \widehat \Wh(\phi^{-1}(x), \kappa),
    \label{eq:whitney_form}
  \end{equation}
  where $\omega_\kappa \in \R$ are the coefficients of the $k$-cochain,
  \begin{equation}
    \widehat \Wh(x, \kappa) = \dd x_{\bi(\kappa)} \prod_{i \in \overline{\bi(\kappa)}} \max \{ 0, 1 - |x_i - I_i(\kappa)| \},
  \end{equation}
  and $I_i(\kappa) \in \Z$ is the element in the degenerate interval. 
\end{defi}
Interestingly, the Whitney map (for simplicial meshes) first appeared in \cite[Eq.~27.12]{whitney1957geometric} but specializes to 
lowest-order Raviart-Thomas~\cite{raviart1977mixed} ($k=2$,$d=3$) and
N{\'e}d{\'e}lec~\cite{nedelec1980mixed} elements (for $k=1$, $d=3$), see~\cite{arnold2014,arnold2006finite}.
Differential forms of type \eqref{eq:whitney_form} are called Whitney forms.

We also define a weighted inner product $\iprod{\cdot}{\cdot}_\phi$ between chains and cochains 
by plugging the Whitney form associated to the $k$-cochain into the current corresponding to the 
$k$-chain. As both are constant on each $k$-cube, a quick calculation shows:
$\iprod{T}{\omega}_\phi = \sum_{\kappa} T_\kappa \omega_\kappa \cH^k(\phi(\kappa))$, where
$\cH^k(\phi(\kappa))$ is simply the volume of the $k$-cube under the mesh spacing $\phi$.

\begin{figure}[t]
  \centering
  \setlength\tabcolsep{0.05cm}
  \begin{tabular}{ccccc}
    \includegraphics[width=0.19\linewidth]{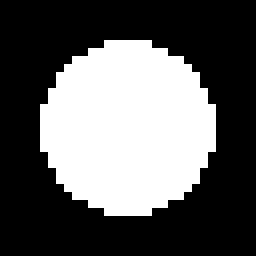}&
    \includegraphics[width=0.19\linewidth]{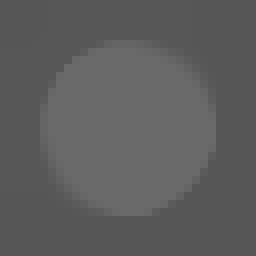}&
    \includegraphics[width=0.19\linewidth]{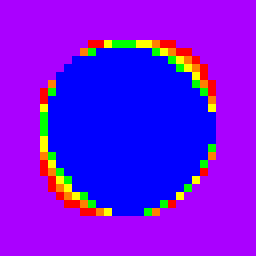} & 
    \includegraphics[width=0.19\linewidth]{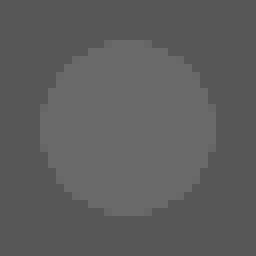} &
    \includegraphics[width=0.19\linewidth]{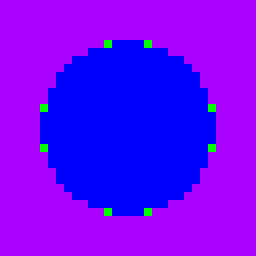} \\[-1mm]
    \small Input & \multicolumn{2}{c}{\small Finite differences} & \multicolumn{2}{c}{\small Discrete exterior calculus}
  \end{tabular}
  \\[2mm]
  \centering
  \setlength\tabcolsep{0.05cm}
  \begin{tabular}{cc}
    \includegraphics[width=0.495\linewidth]{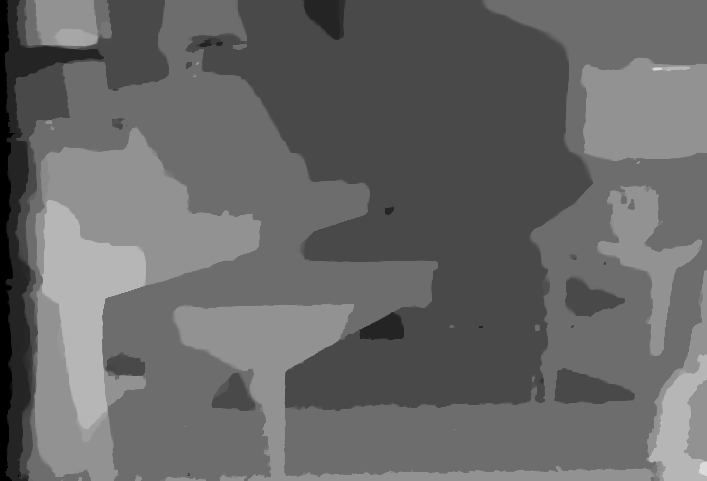}&
    \includegraphics[width=0.495\linewidth]{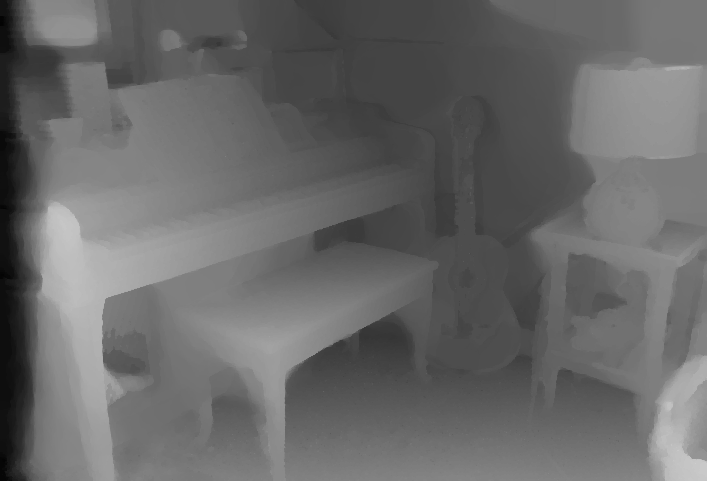}\\[-1mm]
    \small Finite differences 
    &\small Discrete exterior calculus
  \end{tabular}
  \caption{Total variation minimization. \textbf{Top:} The proposed DEC discretization yields solutions with better isotropy and sharper discontinuities. \textbf{Bottom:} In that stereo matching example, we enforce the continuous constraints 
  $\Wh \omega \in \cK$ between the discretization points (here $8$ \emph{labels}), which leads to more precise (sublabel-accurate) solutions compared to the naive finite-difference approach.}
  \label{fig:circle}
\end{figure}

Using the dual representation \eqref{eq:dual_en}, and approximating 
the current by a $k$-chain and the differential forms with $k$-cochains
we arrive at the following finite-dimensional convex-concave saddle-point problem on
$Q \subset \R^{n + N}$:
\begin{equation}
  \begin{aligned}
    &\min_{T \in \cC_n(Q)} \max_{\substack{\omega \in \cC^n(Q)\\ \varphi \in \cC^{n-1}(Q)}} \, \iprod{T}{\omega}_\phi + \iprod{\partial T - S}{\varphi}_\phi, \\
    &\quad \text{subject to } \quad \pi_{1\sharp} T = \mathbf{1}, \Wh \omega \in \cK, \\
    &\quad \text{potentially } \hspace{0.19cm}  \pi_{2\sharp} T = \mathbf{1} \text{ in case } n = N.  
  \end{aligned}
  \label{eq:saddle}
\end{equation}
$S \in \cC_{n-1}(Q)$ is a given boundary datum, for free boundary conditions we replace the inner product $\iprod{S}{\varphi}$ with an indicator function $S : \cC^{n-1} \to \Rext$ forcing $\varphi$ to be zero on the boundary.
The pushforwards $\pi_{1\sharp}$, $\pi_{2\sharp}$ are linear constraints on the coefficients of the $k$-chain $T$.

\begin{figure}[t!]
  \centering
  \setlength\tabcolsep{0.05cm}
  \begin{tabular}{cccc}
    \includegraphics[width=0.24\linewidth]{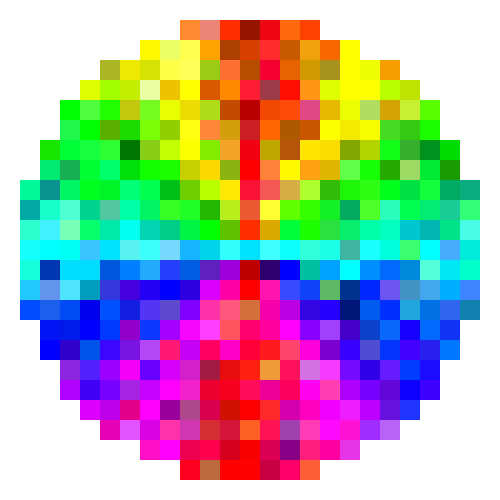}
    &\includegraphics[width=0.24\linewidth]{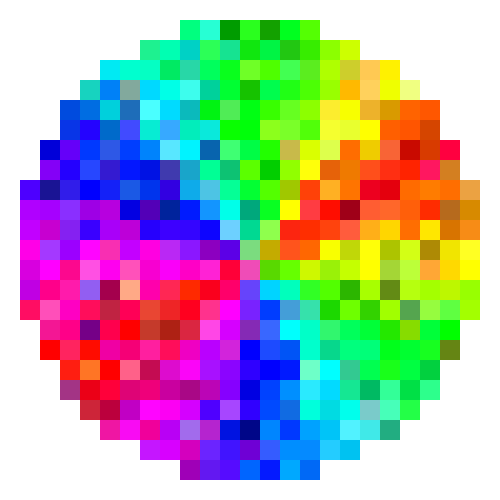}
    &\includegraphics[width=0.24\linewidth]{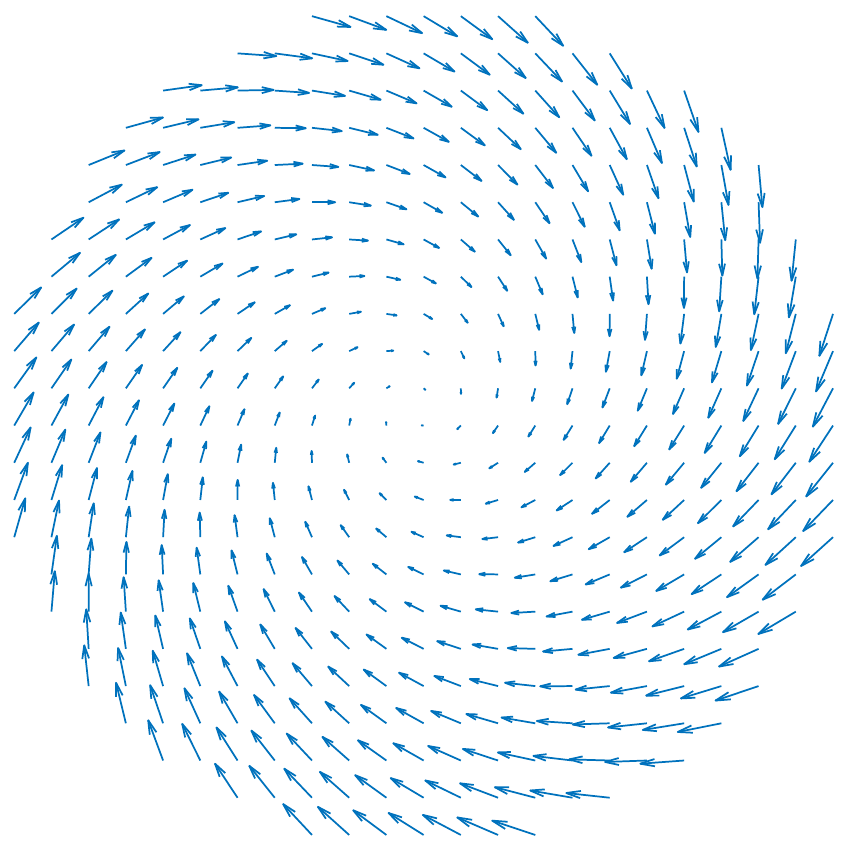}
    &\includegraphics[width=0.24\linewidth]{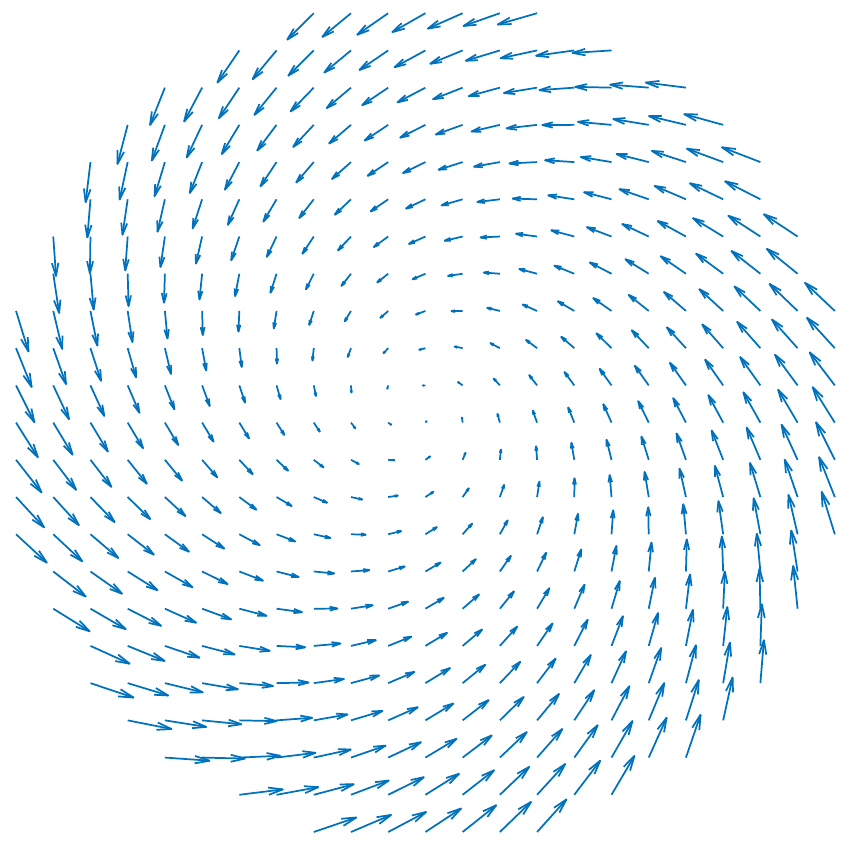} \\
    $X$ 
    &$Y$ 
    &$\id - f$
    &$\id - f^{-1}$\\[2mm]
    \includegraphics[width=0.24\linewidth]{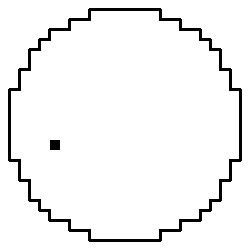}
    &\includegraphics[width=0.24\linewidth]{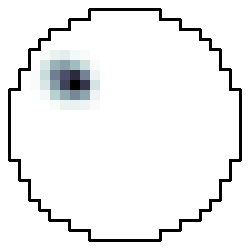}
    &\includegraphics[width=0.24\linewidth]{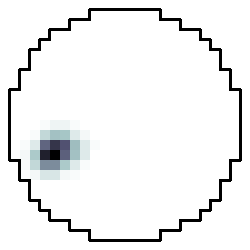} 
    &\includegraphics[width=0.24\linewidth]{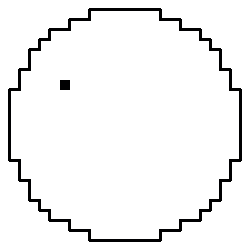}\\
    \multicolumn{2}{c}{left-to-right slice of $\norm{T}$} 
    &\multicolumn{2}{c}{right-to-left slice of $\norm{T}$} \\
  \end{tabular}
  \caption{Global registration of $X$ and $Y$. \textbf{Top:} Our method yields dense pointwise correspondences that are smooth in both directions and correspond to the correct transformation. \textbf{Bottom:} $2$-D slices 
 through the $4$-D density $\norm{T}$ at the single black pixel. We empirically verify (also at the other points) that the current concentrated near a surface, therefore the recovered solution is near the \emph{global minimum} of the original nonconvex problem.}
  \label{fig:corr}
\end{figure}

\section{Numerical Implementation}
\label{sec:numerics}
In practice we solve \eqref{eq:saddle} with the first order primal-dual algorithm \cite{Chambolle-Pock-jmiv11}. 
For the local constraints $\Wh \omega \in \cK$ usually no closed form projection exists.
In some situations ($N=1$) they can be implemented exactly, see~\cite{moellenhoff-laude-cvpr-2016, moellenhoff-iccv-2017}.
In the general setting, we resort to implementing them at 
sampled points.  To enforce the constraint $\Wh \omega \in \cK$ at
samples $Z = \{ z_1, z_2, \hdots \} \subset X \times Y$ we add another primal
variable $\lambda : Z \to \LM_n \R^{n+N}$ and the additional term 
$\sum_{z \in Z} \Psi^{**}(z, \lambda(z)) - \iprod{\lambda(z)}{(\Wh \omega)(z)}$
to the saddle-point formulation \eqref{eq:saddle}. 

Finally, one requires the proximal operator of the perspective function $\Psi^{**}$. 
These can be implemented using epigraphical projections as in \cite{PCBC-SIIMS}.
For an overview over proximal operators of perspective functions we refer to~\cite{combettes2018perspective}.

\subsection{Properties of the Discretization}
As a first example we solve 
the Brachistochrone~\cite{bernoulli}, arguably the first variational approach.
The cost is given by $c(x, y, \xi) = \sqrt{\frac{1 + \xi^2}{2 g y}}$ where $g>0$ is the gravitational constant.
Dirichlet boundary conditions are enforced using the boundary operator.
In Fig.~\ref{fig:illu} we show the resulting current, which concentrates on the graph of 
the closed-form solution to the problem, which is a cycloid. The unlifted result is obtained by taking the 
center of mass of the first component $T^{\bar 0,0}$ of the current by summing over the horizontal edges in the $1$-chain. 
The obtained result nearly coincides with the exact cycloid. Instead, solutions from
MRF approaches would invariably be confined to the vertices or edges of the rather coarse grid.

In Fig.~\ref{fig:circle} we solve total variation regularized problems which corresponds 
to setting $c(x, y, \xi) = \rho(x, y) + \normc{\xi}$ for some data $\rho$. The data 
is either a quadratic or a stereo matching cost in that example. The proposed approach based 
on discrete exterior calculus has better isotropy and leads to sharper discontinuities than the common forward difference approach
used in literature. Furthermore, by enforcing the constraints $\Wh \omega \in \cK$ also between the discretization points 
one can achieve ``sublabel-accurate'' results as demonstrated in the stereo matching example.

\subsection{Global Registration}
As an example of $n > 1$, $N > 1$ with polyconvex regularization, 
we tackle the problem of orientation preserving diffeomorphic registration between two shapes $X, Y \subset \R^2$ with boundary.
We use the cost $c(x, y, \xi) = \left( \rho(x, y) + \varepsilon \right) \sqrt{\det \left( I + \xi^\top \xi \right)}$, 
which penalizes the surface area in the product space and favors local isometry. The parameter $\varepsilon > 0$ 
models the trade-off between data and smoothness. In the example considered in Fig.~\ref{fig:corr} 
the data is given by $\rho(x, y) = \norm{I_1(x) - I_2(y)}$, where $I_1, I_2$ are the shown color images.
A polyconvex extension of the above cost, which is large for non-simple vectors is given by the (weighted) mass norm \eqref{eq:mass}. 
The $4$-D cubical set $Q$ is the product space between the two shapes $X$ and $Y$, which are given as quads (pixels).
We impose the constraints $\Wh \omega \in \cK$ at the $16$ vertices of each four dimensional hypercube. 
The proximal operator of the mass norm is computed as in \cite{windheuser2016convex}. Note that the required
$4 \times 4$ real Schur decomposition can be reduced to a $2 \times 2$ SVD using a few Givens rotations, see \cite{SkewSVD}. 
We further impose $T^{\, \bar0,0} \geq 0$ and $T^{\, 0,\bar0} \geq 0$, and boundary conditions ensure that $\partial X$ is matched to
$\partial Y$. Bijectivity of the matching is encouraged by the pushforward constraints $\pi_{1\sharp} T = \mathbf{1}$, $\pi_{2\sharp} T = \mathbf{1}$.
After solving \eqref{eq:saddle} we obtain the final pointwise correspondences $f : X \to Y$ from the $2$-chain $T \in \cC_2(Q)$ by taking its center of mass.

In Fig.~\ref{fig:corr} we visualize $f(x) = \sum_y y \, \normc{(\Wh T)(x,y)}$, $f^{-1}(y) = \sum_x x \, \normc{(\Wh T)(x,y)}$.
As one can see, the maps $f$ and $f^{-1}$ are smooth and inverse to each other. Despite $n > 1$, $N > 1$, the current apparently concentrated near a surface and
the computed solutions are therefore near the \emph{global optimum} of the original nonconvex problem.

\section{Discussion and Limitations}
\label{sec:discussion}
In this work, we introduced a novel approach to vectorial variational problems 
based on geometric measure theory, along
with a natural discretization using concepts from discrete exterior calculus. 
Though observed in practice, we do not have theoretical guarantees that the minimizing current will concentrate on a surface.
In case of multiple global solutions, one might get a convex combination of minimizers. Some mechanism to select an extreme point of the convex solution set would therefore be desirable.
The main drawback over MRFs, for which very efficient solvers exist~\cite{kappes2013comparative}, is that we had to resort to the generic 
algorithm~\cite{Chambolle-Pock-jmiv11} with $\cO(1/\varepsilon)$ convergence. 
Yet, solutions with high numerical accuracy are typically not required in practice and the algorithm parallelizes well on GPUs.
To conclude, we believe that the present work is a step towards making continuous approaches an attractive alternative 
to MRFs, especially in scenarios where faithfulness to certain geometric properties of 
the underlying continuous model is desirable.

\paragraph{Acknowledgements.} The work was partially supported by
the German Research Foundation (DFG); project 394737018 ``\emph{Functional Lifting 2.0 -- Efficient Convexifications for Imaging and Vision}''.

{\small
\bibliographystyle{ieee}
\bibliography{../../reference.bib}
}

\end{document}